\documentclass[journal]{IEEEtran}
\usepackage{amsmath,graphicx,cite,bm,amssymb,amsthm,enumerate,epsfig,psfrag,cases}
\renewcommand{\(}{\left(}
\renewcommand{\)}{\right)}
\renewcommand{\[}{\left[}
\renewcommand{\]}{\right]}

\newcommand{\m}{\mathbf{m}}

\newcommand{\s}{\mathbf{s}}

\renewcommand{\S}{\mathbf{S}}

\newcommand{\z}{\mathbf{z}}

\newcommand{\R}{\mathbf{R}}

\newcommand{\D}{\mathbf{D}}
\newcommand{\Z}{\mathbf{Z}}

\newcommand{\x}{\mathbf{x}}

\newcommand{\I}{\mathbf{I}}
\newcommand{\J}{\mathbf{J}}

\renewcommand{\P}{\mathbf{P}}
\renewcommand{\t}{\mathbf{t}}
\newcommand{\A}{\mathbf{A}}
\newcommand{\T}{\mathbf{T}}
\newcommand{\U}{\mathbf{U}}
\newcommand{\M}{\mathbf{M}}
\newcommand{\N}{\mathbb{N}}

\renewcommand{\u}{\mathbf{u}}

\newcommand{\Q}{\mathbf{Q}}

\newcommand{\V}{\mathbf{V}}

\newcommand{\Y}{\mathbf{Y}}

\newcommand{\Tr}[1]{{\rm{Tr}}\left(#1\right)}

\newcommand{\End}[1]{{\rm{End}}}

\renewcommand{\log}[1]{{\rm{log}}#1}

\newcommand{\diag}[1]{{\rm{diag}}\left\{#1\right\}}

\renewcommand{\vec}[1]{{\rm{vec}}\(#1\)}
\newcommand{\vech}[1]{{\rm{vech}}\(#1\)}
\newcommand{\mat}[1]{{\rm{mat}}\(#1\)}

\newcommand{\spann}[1]{{\rm{span}}\(#1\)}
\newcommand{\rank}[1]{{\rm{rank}}#1}

\newtheorem{lemma}{Lemma}

\newtheorem{theorem}{Theorem}

\newtheorem{corollary}{Corollary}

\newcommand{\norm}[1]{\left\lVert#1\right\rVert}

\usepackage[ruled,vlined,resetcount]{algorithm2e}
\usepackage{algorithmic,subfig}
\usepackage{fontenc}
\usepackage{inputenc}
\usepackage{epstopdf,pst-node}
\usepackage{mathdots,enumitem}

\begin{document}
\title{Joint Inverse Covariances Estimation with Mutual Linear Structure}
\author{Ilya~Soloveychik and Ami~Wiesel, \\
 Rachel and Selim Benin School of Computer Science and Engineering, The Hebrew University of Jerusalem, Israel
\thanks{This work was partially supported by the Intel Collaboration Research Institute for Computational Intelligence, the Kaete Klausner Scholarship and ISF Grant 786/11.}
\thanks{A part of results was presented at the $23$-th European Signal Processing Conference (EUSIPCO), August 31- September 4, 2015, Nice, France.}
}

\maketitle
\pagenumbering{gobble}
\begin{abstract}
We consider the problem of joint estimation of structured inverse covariance matrices. We perform the estimation using groups of measurements with different covariances of the same unknown structure. Assuming the inverse covariances to span a low dimensional linear subspace in the space of symmetric matrices, our aim is to determine this structure. It is then utilized to improve the estimation of the inverse covariances. We propose a novel optimization algorithm discovering and exploiting the underlying structure and provide its efficient implementation. Numerical simulations are presented to illustrate the performance benefits of the proposed algorithm.
\end{abstract}


\IEEEpeerreviewmaketitle

\section{Introduction}
Large scale covariance and inverse covariance estimation using a small number of measurements is a fundamental problem in modern multivariate statistics. In many applications, e.g., linear array processing, climatology, spectroscopy, and longitudinal data analysis statistical properties of variables can be related due to natural physical features of the systems. Such relations often imply linear structure in the corresponding population covariance or inverse covariance matrices.

In this paper we focus on structured inverse covariance (concentration) matrix estimation. There are plenty of examples of inverse covariance estimation with linear structure. A partial list includes banded \cite{bickel2008regularized, kavcic2000matrices, asif2005block}, circulant \cite{dembo1989embedding, cai2013optimal}, sparse (graphical) models \cite{friedman2008sparse,banerjee2008model}, etc. An important common feature of these works is that they consider a single and static environment where the structure of the true concentration matrix, or at least the class of structures, as in the sparse case, is known in advance. Often, this is not the case and techniques are needed to learn the structure from the observations. A typical approach is to consider multiple datasets sharing a similar structure but non homogeneous environments \cite{guo2011joint, besson2008covariance, bidon2008bayesian, aubry2014exploiting, lee2015joint}. This is, for example, the case in covariance estimation for classification across multiple classes \cite{danaher2014joint}. A related problem addresses tracking a time varying covariance throughout a stream of data \cite{wiesel2013time, moulines2005recursive}, where it is assumed that the structure changes at a slower rate than the covariances themselves \cite{ahmed2009recovering}. Here too, it is natural to divide this stream of data into independent blocks of measurements.

Our goal is to first rigorously state the problem of joint concentration matrices estimation with linear structure and derive its fundamental lower performance bounds. Secondly, we propose and analyze a new algorithm of learning and exploring this structure to improve estimation of the inverse covariance matrices. More exactly, given a few groups of measurements having different covariance matrices each, our target is to determine the underlying low dimensional linear space containing or approximately containing the concentration matrices of all the groups. The discovered subspace can be further used to improve the inverse covariance estimation by projecting any unconstrained estimator on it. Most of the previous works considered particular cases of this method, e.g. factor models, entry-wise linear structures like in sparse and banded cases, or specific patterns like in banded, circulant and other models. We propose a new generic algorithm based on the idea of Lasso regression, \cite{tibshirani1996regression}. Namely, our method seeks for minimizing the joint negative log-likelihood penalized by the dimensionality of the subspace containing the inverse covariances. 

In \cite{soloveychik2015joint} we considered a similar setting, where the covariance matrices, rather than their inverses, share the same linear structure. There we approached the problem of low dimensional covariance estimation using the Truncated SVD (TSVD) technique applied to the sample covariance matrices (SCM-s) of the groups of measurements. The reason we do not use the TSVD approach in the case of concentrations is that it would require inversion of the SCM-s, which becomes problematic when the number of samples in groups is relatively small. The algorithm we propose here avoids this restriction and demonstrates good performance even when the amount of measurements is scarce. Due to its Lasso-type construction, our algorithm naturally involves regularization parameter, which can be tuned using cross-validation schemes. Instead, we propose a simple closed form and data dependent tuning approach.

The rest of the text is organized as following: first we introduce notations, state the problem and illustrate examples. Then we derive the lower performance bound, propose our Joint Inverse Covariance Estimation (JICE) technique and an efficient numerical method computing it. In addition, we provide an upper performance bound for the JICE algorithm and demonstrate its application to the regularization parameter tuning. Finally, we perform numerical simulations demonstrating the advantages of the proposed algorithm.

Given $p \in \N$, denote by $\mathcal{S}(p)$ the $l = \frac{p(p+1)}{2}$ dimensional linear space of $p \times p$ symmetric real matrices. $\I_d$ stands for the $d \times d$ identity matrix. For a square matrix $\M$, its Moore-Penrose generalized inverse is denoted by $\M^\dagger$. For any two matrices $\M$ and $\P$ we denote by $\M\otimes\P$ their tensor (Kronecker) product. $\sigma_1(\R) \geqslant \dots \geqslant \sigma_r(\R) \geqslant 0$ stand for the singular values of a rectangular matrix $\R$ of rank not greater than $r$. If $\M \in \mathcal{S}(p)$, we denote its eigenvalues by $\lambda_1(\M) \geqslant \dots \geqslant \lambda_p(\M)$ and write $\M \succ 0$ if $\lambda_p(\M) > 0$. $\norm{\cdot}_F$ denotes the Frobenius, $\norm{\cdot}_2$ - the spectral and $\norm{\cdot}_*$ - the nuclear (trace) norms of matrices, correspondingly. For any symmetric matrix $\S$, $\s = \vech{\S}$ is a vector obtained by stacking the columns of the lower triangular part of $\S$ into a single column. In addition, given an $l$ dimensional column vector $\m$ we denote by $\mat{\m}$ the inverse operator constructing a $p \times p$ symmetric matrix such that $\vech{\mat{\m}}=\m$. Due to this natural linear bijection below we often consider subsets of $\mathcal{S}(p)$ as subsets of $\mathbb{R}^l$. In addition, let $\vec{\S}$ be a $p^2$ dimensional vector obtained by stacking the columns of $\S$, and denote by $\mathcal{I}$ its indices corresponding to the related elements of $\vech{\S}$. Finally, we write $\x \sim \mathcal{N}(\bm{0},\Q)$ when $\x$ is a centered normal random vector with covariance $\Q$.

\section{Problem Formulation and Examples}
\label{pr_f}
Consider a heterogeneous Gaussian model, namely, assume we are given $K \geqslant l = \frac{p(p+1)}{2}$ independent groups of real $p$ dimensional normal random vectors
\begin{equation}
\x_k^i \sim \mathcal{N}(\bm{0},\Q_k),\;\; i=1,\dots,n,\;\; k=1,\dots,K,
\end{equation}
with $n$ i.i.d. (independent and identically distributed) samples in each group and covariances
\begin{equation}
\Q_k = \mathbb{E}[\x_k\x_k^T],\;\; k=1,\dots,K.
\end{equation}
We assume that the inverse covariances
\begin{equation}
\T_k = \Q_k^{-1} \succ 0,\;\; k=1,\dots,K,
\label{q_def}
\end{equation}
exist and span an $r$ dimensional linear subspace of $\mathcal{S}(p)$. Our main goal is to estimate this subspace and use it to improve the concentration matrices estimation.

Let us list a few common linear subspaces which naturally appear in typical signal processing applications.
\begin{itemize}
[leftmargin=*]
\item {\bf{Diagonal}}:
The simplest example of a structured concentration matrix is a diagonal matrix. This is often the case when the noise vectors are uncorrelated or can be assumed such with great precision. In this case $r=p$. Note that the diagonal structure remains unaltered under inversion, thus making diagonal concentration equivalent to the diagonal covariance case, considered in \cite{soloveychik2015joint}.

\item {\bf{Banded}}: It is often reasonable to assume that the non-neighboring elements of a normal random vector are conditionally independent given all the other elements. Specifically, claiming that $i$-th element of the random vector is conditionally independent on the $h$-th if $|i-h|>b$ leads to the $b$-banded inverse covariance structure. The subspace of symmetric $b$-banded matrices constitutes an $r=\frac{(2p-b)(b+1)}{2}$ dimensional subspace inside $\mathcal{S}(p)$. Banded inverse covariance matrices are ubiquitous in graphical models, \cite{friedman2008sparse,banerjee2008model}.

\item {\bf{Circulant}}:
\label{circ_def}
The next common type of structured concentration matrices are symmetric circulant matrices, defined as
\begin{equation}
\T =
 \begin{pmatrix}
  t_1 & t_2 & t_3 & \dots & t_p \\
  t_p & t_1 & t_2 & \dots & t_{p-1}\\
  \vdots & \vdots & \vdots & \ddots & \vdots \\
  t_2 & t_3 & t_4 & \dots & t_1
  \end{pmatrix},
\label{circ_struct}
\end{equation}
with the natural symmetry conditions such as $t_p = t_2,$ etc. Such matrices are typically used as approximations to Toeplitz matrices which are associated with signals obeying periodic stochastic properties, e.g. the yearly variation of temperature in a particular location. A special case of such processes are classical stationary processes, which are ubiquitous in engineering, \cite{dembo1989embedding, cai2013optimal}. Symmetric circulant matrices constitute an $r=p/2$ dimensional subspace if $p$ is even and $(p+1)/2$ if it is odd. Interestingly, like in the diagonal case, this structure does not change under inversion, \cite{soloveychik2014groups}, making the estimation problems in covariances and concentrations analogous, \cite{soloveychik2015joint}.

\item {\bf{Sparse}}:
Sparse inverse covariance models generalize banded structures and are very common. In multivariate Gaussian distributions zero entries of the concentration reveal conditional independences. When graph representation is utilized to express the relations between the variates, zeros in inverse covariance are translated to missing edges in the graph making it sparse. Recently, Gaussian graphical models have attracted considerable attention due to developments in biology, medicine, neuroscience, compressed sensing and many other areas, \cite{friedman2008sparse,banerjee2008model,lauritzen1996graphical,wiesel2010covariance}. Unlike the previous examples, most of the literature treating graphical models does not assume the graph structure to be known in advance and aims at learning this structure.

\end{itemize}
In the following it will be convenient to use a single matrix notation for the multiple concentration matrices
\begin{align}
\t_k &= \vech{\T_k},\;\; k=1,\dots,K,\\
\Y &= [\t_1,\dots,\t_K].
\end{align}
Using this notation, the prior subspace knowledge discussed above is equivalent to a low-rank constraint
\begin{equation}
\Y = \U\Z,
\label{new_param}
\end{equation}
where $\U \in \mathbb{R}^{l \times r}$ and $\Z = [\z_1,\dots,\z_K] \in \mathbb{R}^{r \times K}$. Essentially our problem reduces to estimation of $\Y$ assuming it is low-rank. In the analysis we will assume $r$ is known in advance, but the algorithm we propose recovers it from the data. Note that unlike \cite{soloveychik2015joint}, we assume a linear model in (\ref{new_param}) and not affine, meaning that we do not separate the mean explicitly. The reason is that the additional intercept term complicated the derivations and have not led to significant gains in our numerical simulations.

\section{Lower Performance Bounds}
We begin by addressing the inherent lower performance bounds. For this purpose we use the Cramer-Rao Bound ($\mathbf{CRB}$) to lower bound the Mean Squared Error ($\mathbf{MSE}$) of any unbiased estimator $\widehat{\Y}$ of $\Y$, defined as
\begin{equation}
\mathbf{MSE} = \mathbb{E}\[\norm{\widehat{\Y} - \Y}_F^2\].
\end{equation}

The $\mathbf{MSE}$ is bounded from below by the trace of the corresponding $\mathbf{CRB}$ matrix. To compute this matrix, for each $i$ we stack the measurements $\x_k^i$ into a single vector
\begin{equation}
\x^i  =
\begin{pmatrix}
 \x_1^i \\
 \vdots \\
 \x_K^i
\end{pmatrix}
\sim \mathcal{N}(\bm{0},\T^{-1}),\;\; i=1,\dots,n.
\end{equation}
where the block-diagonal matrix $\T$ is defined as
\begin{multline}
\T(\U,\Z) = \diag{\T_1,\dots,\T_k} \\
= \diag{\mat{\U\z_1},\dots,\mat{\U\z_K}}.
\end{multline}
Here the operator $\diag{\T_1,\dots,\T_K}$ returns a block-diagonal matrix of size $pK \times pK$ with $\T_k$-s as its diagonal blocks. The Jacobian matrix of this parametrization reads as
\begin{align}
&\J = \frac{\partial\T}{\partial (\U,\Z)} =
 \begin{pmatrix}
  \frac{\partial \t_1}{\partial \U} & \frac{\partial \t_1}{\partial \z_1} & 0 & \dots & 0 \\
  \frac{\partial \t_2}{\partial \U} & 0 & \frac{\partial \t_2}{\partial \z_2} & \dots & 0 \\
  \vdots & \vdots & \vdots & \ddots & \vdots \\
  \frac{\partial \t_K}{\partial \U} & 0 & 0 & \dots & \frac{\partial \t_K}{\partial \z_K} \\
  \end{pmatrix} \nonumber\\
&=
 \begin{pmatrix}
  \z_1^T \otimes \I_l & \U & 0 & \dots & 0 \\
  \z_2^T \otimes \I_l & 0 & \U & \dots & 0 \\
  \vdots & \vdots & \vdots & \ddots & \vdots \\
  \z_K^T \otimes \I_l & 0 & 0 & \dots & \U \\
  \end{pmatrix}
\in \mathbb{R}^{lK \times (lr+Kr)},
\end{align}
where we have used the following notation:
\begin{equation}
\frac{\partial \t_k}{\partial \U} =
\[\frac{\partial \t_k}{\partial \u_1} \; \frac{\partial \t_k}{\partial \u_2} \; \dots \; \frac{\partial \t_k}{\partial \u_r}\],
\end{equation}
and the formulae
\begin{equation}
\frac{\partial \t_k}{\partial \u_j} = \frac{\partial \U\z_k}{\partial \u_j} = z^j_k\I_l,\quad
\frac{\partial \t_k}{\partial \z_k} = \frac{\partial \U\z_k}{\partial \z_k} = \U.
\end{equation}

\begin{lemma}(Lemma 1 from \cite{soloveychik2015joint})
\begin{equation}
\rank(\J) = lr + Kr - r^2
\end{equation}
\end{lemma}

This lemma implies that $\J$ is rank deficient, as
\begin{equation}
\rank(\J) = lr + Kr - r^2 \leqslant \min[lK, lr+Kr],
\end{equation}
reflecting the fact that the parametrization of $\T$ or $\Y$ by the pair $(\U,\Z)$ is unidentifiable. Indeed for any invertible matrix $\A$, the pair $(\U\A, \A^{-1}\Z)$ fits as good. Due to this ambiguity the matrix $\mathbf{FIM}(\U,\Z)$ is singular and in order to compute the $\mathbf{CRB}$ we use the Moore-Penrose pseudo-inverse of $\mathbf{FIM}(\U,\Z)$ instead of inverse, as justified by \cite{li2012interpretation}. Given $n$ i.i.d. samples $\x^i,i=1,\dots,n$, we obtain
\begin{equation}
\mathbf{CRB} = \frac{1}{n}\J\;\mathbf{FIM}(\U,\Z)^\dagger\J^T.
\end{equation}
For the Gaussian population the matrix $\mathbf{FIM}(\U,\Z)$ is given by
\begin{equation}
\mathbf{FIM}(\U,\Z) = \frac{1}{2}\J^T \diag{\[\T_k^{-1} \otimes \T_k^{-1}\]_{\mathcal{I},\mathcal{I}}}\J,
\end{equation}
where $\[\M\]_{\mathcal{I},\mathcal{I}}$ is the square submatrix of $\M$ corresponding to the index set $\mathcal{I}$, defined in the notations section. The bound on the $\mathbf{MSE}$ is therefore given by
\begin{align}
\mathbf{MSE} &\geqslant \Tr{\mathbf{CRB}} = \frac{1}{n}\Tr{\mathbf{FIM}(\U,\Z)^\dagger \J^T\J} \nonumber\\
&= \frac{2}{n}\Tr{\[\J^T \diag{\[\T_k^{-1} \otimes \T_k^{-1}\]_{\mathcal{I},\mathcal{I}}}\J\]^\dagger\J^T\J}.
\label{crb_e}
\end{align}
To get more intuition on the dependence of the $\mathbf{MSE}$ on the model parameters, we bound it from below. Denote
\begin{equation}
\underline\lambda = \min_k\lambda_p(\T_k),\quad \overline\lambda = \max_k \lambda_1(\T_k),
\end{equation}
to get the bound
\begin{align}
\mathbf{MSE} &\geqslant \frac{2\underline\lambda^2}{n}\Tr{\[\J^T\J\]^\dagger \J^T\J} \nonumber \\
&= \frac{2\underline\lambda^2}{n}\rank\;{\J} = \frac{2\underline\lambda^2}{n}(lr + Kr - r^2).
\label{crb_b}
\end{align}
As expected, the dependence on the model parameters is similar to that obtained in \cite{soloveychik2015joint} for the joint structured covariance estimation and in \cite{tang2011lower} for the problem of low-rank matrix reconstruction.

\section{JICE Algorithm}
\label{jice_alg}
We now introduce a novel algorithm based on convex penalized likelihood relaxation, whose purpose is to estimate the concentrations $\T_1,\dots,\T_K$ utilizing the prior that they belong to a low dimensional subspace. We, therefore, refer to it as the Joint Inverse Covariance Estimation (JICE) algorithm. The most natural brute force approach would be to form the $K$ SCM-s
\begin{equation}
\S_k = \frac{1}{n} \sum_{i=1}^n \x_i^k\x_i^{kT},\;\; k=1,\dots,K,
\label{scmd}
\end{equation}
invert them and fit by a subspace of a small dimension. This can be done, for example, by the means of Principal Component Analysis (PCA). Such approach, applied to the SCM-s (rather than to their inverses), was proposed in \cite{soloveychik2015joint} to treat the problem of joint covariance estimation with mutual linear structure. When the number of samples $n$ in each group is smaller than or even close to the dimension $p$ (the scenario we are mostly interested in), the inversion of the SCM-s would be impossible due to rank deficiency or would approximate the true inverse covariances poorly, thus causing the proposed PCA algorithm to fail. Instead, we suggest a different approach based on regularized likelihood maximization and its efficient implementation.

\subsection{The Basic Algorithm}
\label{bas_alg}
Our aim is to estimate the true inverse covariances, while simultaneously trying to keep the dimension of the space spanned by the estimators small. For this purpose we suggest the following regularized average log-likelihood optimization program
\begin{multline}
[\widehat{\T}_1,\dots,\widehat{\T}_K] \\
= \operatornamewithlimits{argmin}_{\widetilde{\T}_1,\dots,\widetilde{\T}_K} \frac{1}{K}\sum_{k=1}^K \[- \log|\widetilde{\T}_k|+\Tr{\S_k\widetilde{\T}_k}\] + \eta\norm{\widetilde{\Y}}_*,
\label{main_alg}
\end{multline}
where
\begin{equation}
\widetilde{\Y} = \[\vech{\widetilde{\T}_1},\dots,\vech{\widetilde{\T}_K}\],
\end{equation}
and $\eta$ is a regularization parameter. Program (\ref{main_alg}) assumes the samples are Gaussian and aims at minimizing the average negative log-likelihood of the $K$ groups of measurements and simultaneously enforces joint low dimensional structure on the concentration matrices. The latter is achieved by penalizing the nuclear norm of $\widetilde{\Y}$. Nuclear norm is a convex envelope of the counting measure on singular values of a matrix. Therefore, intuitively the purpose of the penalty term in (\ref{main_alg}) is to decrease the number of non-zero singular values, \cite{candes2009exact}. The optimization problem (\ref{main_alg}) can be viewed as the analog of the Lasso estimator \cite{tibshirani1996regression}, tailored to low-rank matrices as opposed to sparse vectors. Note that the program at hand is convex, and can be treated by many standard numerical solver, such as CVX, \cite{cvx, gb08}.

Usually the value of the regularization parameter $\eta$ is specified by the statistician based on his expert prior knowledge on the estimated quantities and the desired balance between bias and variance of the estimate. As a part of the theoretical results, in the next section we recommend a particular choice of this parameter guarantying closeness of the estimate to the unknown true inverse covariance matrix in Frobenius norm.

\subsection{Extensions of the JICE}
The target function (\ref{main_alg}) of the proposed JICE algorithm assumes the data to be normally distributed and is, therefore, based on the Gaussian negative log-likelihoods. However, a natural extension can be easily obtained by replacing these log-likelihoods by any other convex functions. More generally, the same relaxed penalized likelihood optimization technique for joint estimation of matrices with a common linear structure can be applied to any other estimator given as a solution to a convex optimization program. For example, a straight forward extension to a robust structured covariance estimation is achieved by adding a regularizing nuclear penalty term to a joint target constructed from multiple convexly relaxed M-estimators. An examples of such relaxation is the COCA estimator, \cite{soloveychik2014tyler}.

As an additional option, when some convex prior structure $\mathcal{S}$ on the unknown inverse covariances is given, this information can be easily incorporated into the target by imposing a set of natural convex constraints $\widetilde{\T}_k \in \mathcal{S},\;k=1,\dots,K$.

\subsection{Bias Removal}
The JICE algorithm suffers from bias since we replace the intractable rank counting penalty on the estimated matrix $\Y$ by the nuclear norm, which is a convex proxy of the counting measure on singular values. We suggest an additional step aimed at removing this bias. Denote the vectorized solution of (\ref{main_alg}) by
\begin{equation*}
\widehat{\Y} = \[\vech{\widehat{\T}_1},\dots,\vech{\widehat{\T}_1}\].
\end{equation*}
Analogously to the original parametrization (\ref{new_param}), consider the decomposition
\begin{equation}
\widehat{\Y} = \widehat{\U}\widehat{\Z},
\end{equation}
where $\widehat{\U} \in \mathbb{R}^{l\times s}$ has orthonormal columns and $\widehat{\Z} \in \mathbb{R}^{s\times K}$. Such decomposition naturally suggests treating the columns of $\widehat{\U}$ as the basis vectors of the approximate low dimensional subspace of $\mathcal{S}(p)$. We get an improved concentration matrices estimator by minimizing the negative log-likelihoods over the subspace spanned by $\widehat{\U}$
\begin{multline}
\widehat{\T}_k' =  \operatornamewithlimits{argmin}_{\widetilde{\T}_k \subset \spann{\widehat{\U}}} - \log|\widetilde{\T}_k|+\Tr{\S_k\widetilde{\T}_k}, \\ k=1,\dots,K.
\label{unb_alg}
\end{multline}
Remarkably, this additional bias removal step requires solving $K$ decoupled programs that can be performed in parallel.

\subsection{ADMM Implementation}
In spite of the fact the JICE algorithm described above can be solved by any standard general purpose convex numerical library, one of our goals was to develop a stand alone implementation. Following \cite{boyd2011distributed}, we propose an efficient solution to (\ref{main_alg}) based on Alternating Direction Method of Multipliers (ADMM). The idea behind the ADMM is to optimize the augmented target
\begin{multline}
f(\widetilde{\T}_1,\dots,\widetilde{\T}_K, \widetilde{\Y}) = \frac{1}{K}\sum_{k=1}^K \[- \log|\widetilde{\T}_k|+\Tr{\S_k\widetilde{\T}_k}\] \\ + \eta\norm{\widetilde{\Y}}_* 
+ \frac{\rho}{2}\frac{1}{K}\sum_{k=1}^K \norm{\vech{\widetilde{\T}_k} - \widetilde{\Y}_{:,k}}^2,
\label{aug_target}
\end{multline}
with respect to its parameters $\widetilde{\T}_1,\dots,\widetilde{\T}_K, \widetilde{\Y}$ under the constraint
\begin{equation}
\text{s.t.}\quad\quad\;\vech{\widetilde{\T}_k} = \widetilde{\Y}_{:,k},\;\; k=1,\dots,K.\quad\quad\quad
\end{equation}
Here $\widetilde{\Y}_{:,k}$ denotes the $k$-th column of $\widetilde{\Y}$ and the constraints enforce consensus between the variables (note that when the constrains are satisfied, the target coincides with the original one, (\ref{main_alg})). The solution is obtained via introduction of dual variables $\U_k \in \mathbb{R}^l$ and follows the iterative scheme (we replace iteration indices by arrows to simplify notations)
\begin{align}
&\widetilde{\T}_k \leftarrow \operatornamewithlimits{argmin}_{\overline{\T}_k} - \log|\overline\T_k| +\Tr{\S_k\overline\T_k} \nonumber\\
& + \frac{\rho}{2}\norm{\vech{\overline\T_k} - \widetilde{\Y}_{:,k}+\U_k}^2,\; k=1,\dots,K, \label{it_1}\\
&\widetilde{\Y} \leftarrow \operatornamewithlimits{argmin}_{\bar{\Y}} \eta\norm{\bar{\Y}}_* + \frac{\rho}{2K}\sum_{k=1}^K\norm{\vech{\widetilde{\T}_k} - \bar{\Y}_{:,k}+\U_k}^2, \label{it_2}\\
&\U_k \leftarrow \U_k + \vech{\widetilde{\T}_k} - \widetilde{\Y}_{:,k},\; k=1,\dots,K.
\end{align}

The main advantage of the proposed ADMM technique is that both updates in (\ref{it_1}) and (\ref{it_2}) have easily computable closed form solutions. Indeed, the first order optimality conditions for the $\widetilde{\T}_k$ update target (\ref{it_1}) yield
\begin{equation}
\rho\overline\T_k - \overline\T_k^{-1} = \rho\;\mat{\widetilde{\Y}_{:,k} - \U_k} -\S_k.
\label{f_u}
\end{equation}
Denote by $\D_k$ the orthogonal matrix diagonalizing the right-hand side of (\ref{f_u}), then
\begin{equation}
\D_k^T\(\rho\;\mat{\widetilde{\Y}_{:,k} - \U_k} -\S_k\)\D_k = \bm\Lambda_k.
\end{equation}
Multiply (\ref{f_u}) by $\D_k^T$ on the left and by $\D_k$ on the right to get
\begin{equation}
\rho\underline\T_k - \underline\T_k^{-1} = \bm\Lambda_k,
\end{equation}
where $\underline\T_k = \D_k^T\overline\T_k\D_k$. We can now construct a diagonal solution of this equation by solving the $p$ quadratic equations involving the diagonal entries of $\underline\T_k$ of the form
\begin{equation}
\rho\underline t_k^j - \frac{1}{\underline t_k^j} = \lambda_k^j,\;\; j=1,\dots,p,
\end{equation}
where $\underline\T_k = \diag{\underline t_k^1,\dots,\underline t_k^p}$ and $\bm\Lambda_k = \diag{\lambda_k^1,\dots,\lambda_k^p}$. Solving these quadratic equations and choosing positive roots, yields
\begin{equation}
\underline t_k = \frac{1}{2\rho}\(\lambda_k + \sqrt{\lambda_k^2+4\rho}\).
\end{equation}
and finally,
\begin{equation}
\underline\T_k = \frac{1}{2\rho}\(\bm\Lambda_k + \sqrt{\bm\Lambda_k^2+4\rho\I}\).
\end{equation}
The corresponding iteration becomes
\begin{equation}
\widetilde{\T}_k \leftarrow \frac{1}{2\rho}\D_k^T\(\bm\Lambda_k + \sqrt{\bm\Lambda_k^2+4\rho\I}\)\D_k,
\end{equation}
which is guaranteed to be a positive definite matrix.

\begin{algorithm}[t]
 \caption{The ADMM Implementation of the JICE}
 \begin{algorithmic}[1]
 \label{algo}
 \renewcommand{\algorithmicrequire}{\textbf{Input:}}
 \renewcommand{\algorithmicensure}{\textbf{Output:}}
 \renewcommand{\algorithmicfor}{\textbf{parfor}}
 \REQUIRE $\S_k,\; \rho,\; \eta,\; \varepsilon,$
 \ENSURE  $\widehat{\T}_k,$
 \\ \textit{Initialization} : $\U \leftarrow 0,\; f_{\text{previous}} = + \infty,\; f_{\text{current}} = 0$

\REPEAT
   \FOR{$k=1\colon K$}
\STATE \text{decompose:   } $\rho\;\mat{\widehat{\Y}_{:,k} - \U_k} -\S_k = \D_k\bm\Lambda_k\D_k^T,$
\STATE $\widehat{\T}_k \leftarrow \frac{1}{2\rho}\D_k^T\(\bm\Lambda_k + \sqrt{\bm\Lambda_k^2+4\rho\I}\)\D_k,$
\ENDFOR
\STATE 
$\widehat{\Y} \leftarrow S_{K\eta/\rho}\(\[\vech{\widehat{\T}_1}+\U_1,..,\vech{\widehat{\T}_K}+\U_K\]\),$
\STATE 
$\U \leftarrow \U + \[\vech{\widehat{\T}_1},\dots,\vech{\widehat{\T}_K}\] - \widehat{\Y},$
\STATE $f_{\text{previous}} \leftarrow f_{\text{current}},$
\STATE $f_{\text{current}} \leftarrow f(\widehat{\T}_1,\dots,\widehat{\T}_K, \widehat{\Y}),$
\UNTIL{$|f_{\text{current}} - f_{\text{previous}}| \geqslant \varepsilon$}

\item[]
\renewcommand{\algorithmicrequire}{\textbf{Bias Removal Step}}
\REQUIRE
\STATE \text{decompose:   }  $\widehat{\Y} = \widehat{\U}\widehat{\Z}$,
\FOR{$k=1\colon K$}
\STATE $\widehat{\T}_k \leftarrow \operatornamewithlimits{argmin}_{\widetilde{\T}_k \subset \spann{\widehat{\U}}} - \log|\widetilde{\T}_k|+\Tr{\S_k\widetilde{\T}_k}$.
\ENDFOR
 \end{algorithmic} 
 \end{algorithm}

In order to derive a closed form solution for the second iterative step (\ref{it_2}), we introduce the matrix singular value soft thresholding operator. Given a matrix $\M$ with the SVD
\begin{equation}
\M = \U\begin{pmatrix}
  \sigma_1 & 0 & \dots \\
  0 & \sigma_2 & \dots \\
  \vdots & \vdots & \ddots \\
  \end{pmatrix}\V^T,
\end{equation}
the singular values soft thresholding operator is defined as
\begin{equation}
S_{\varepsilon}(\M) = \U\begin{pmatrix}
  \max[\sigma_1-\varepsilon,0] & 0 & \dots \\
  0 & \max[\sigma_2-\varepsilon,0] & \dots \\
  \vdots & \vdots & \ddots \\
  \end{pmatrix}\V^T.
\end{equation}
It can be easily shown that the solution to (\ref{it_2}) is given by
\begin{align}
&\widetilde{\Y} \leftarrow +S_{K\eta/\rho}\(\[\vech{\widetilde{\T}_1}+\U_1,\dots,\vech{\widetilde{\T}_K}+\U_K\]\) \nonumber.
\end{align}
A pseudo-code of the proposed JICE algorithm is given in the table Algorithm \ref{algo}, where we denote
\begin{equation}
\U = \[\U_1,\dots,\U_K\].
\end{equation}
Note that the main loop of the algorithm involves two stages in each iteration, corresponding to the estimation and structure learning steps, described above. On the first approximation stage the values of $\widetilde{\T}_k$-s are updated, which can be done in parallel and is denoted in the code by the $\mathbf{parfor}$ loop. On the second stage, the code gathers the new $\widetilde{\T}_k$ values and updates $\widetilde{\Y}$ and $\U$, which can be viewed as the structure learning step. 

Under reasonable conditions, the general ADMM algorithm is claimed \cite{boyd2011distributed} to obtain the following convergence properties: the dual variable and the objective converge to their true values, whether convergence of the primal variables is not guaranteed. Under additional assumptions, primal variables converge to the solution, as well. In addition, as it is mentioned in \cite{boyd2011distributed}, the ADMM can be quite slow to converge to high accuracy. Fortunately, in our problem modest accuracy is sufficient to obtain a good estimator. Our numerical simulations demonstrated that the most reliable stopping criterion is based on the convergence of the objective function value, which is implemented in Algorithm \ref{algo}. 

Another important question is the choice of the penalty parameter $\rho$. In our implementation we always took it constant $\rho=1$, but many different approaches of tuning this parameter exist in literature. See \cite{boyd2011distributed} for a more full treatment of this issue and a list of references.

\subsection{Upper Performance Bound}
\label{upper_bound_sec}
In this section we provide a high probability performance bound for the proposed JICE algorithm. In the derivation presented here we follow the technique used by \cite{bickel2008regularized,ravikumar2011high, rothman2008sparse, soloveychik2014non}. The method consists in showing that the extremal point of the target (\ref{main_alg}) lies within some small vicinity of the true parameter value with high probability. At first, the target is approximated by Taylor's polynomial with a second order remainder in the Lagrange form. Then concentration of measure phenomenon is utilized to obtain high probability bounds on the derivatives and tune the regularization parameter. Finally, it is shown that the regularized solution must belong to some small ball around the true parameter value. The radius of this ball plays the role of the desired error bound with the probability guaranteed by the first step. Denote
\begin{equation}
\Q = \[\vec{\Q_1},\dots,\vec{\Q_K}\],
\end{equation}
and
\begin{equation}
\S = \[\vec{\S_1},\dots,\vec{\S_K}\],
\end{equation}
where $\S_k,\;k=1,\dots,K$ are the SCM-s defined in (\ref{scmd}).

\begin{theorem}
\label{jice_b}
Assume that $n \geqslant p$ and $\eta \geqslant \frac{\norm{\S-\Q}_2}{K}$, then with probability at least $1-K\(2^{-n+p-1}-e^{-c_1n}\)$,
\begin{equation}
\norm{\widehat{\Y} - \Y}_F \leqslant \frac{4\overline\lambda^2\sqrt{2r}K\eta}{c_2(1-\sqrt{(p-1)/n})^4},
\end{equation}
where $c_1$ and $c_2$ do not depend on $K,p,n$ or $r$.
\end{theorem}
\begin{proof}
The proof can be found in the Appendix.
\end{proof}
As expected, this bound shows that the Frobenius norm of the error is proportional to the squared root of the rank rather than the dimension.

\subsection{Regularization Parameter Choice}
One of the main ingredients of the algorithm and the bound guaranteed by Theorem \ref{jice_b} is the condition on the regularization parameter $\eta$. Its optimal choice yielding the best performance characteristics of the algorithm depends on unobservable quantities. Therefore, in general one might need to tune the regularization parameter through cross validation \cite{bickel2006regularization}. Below we show how it can be calculated based on the given measurements. Let us start by estimating the norm $\norm{\S-\Q}_2$.
\begin{lemma}(Lemma 4 from \cite{soloveychik2015joint})
\label{r_lem}
With probability at least $1-2e^{-c_3p}$,
\begin{equation}
\norm{\S-\Q}_2 \leqslant \frac{1}{\underline\lambda\sqrt{n}}\(\sqrt{2K} +  c_4\sqrt{p^2}\),
\label{lle}
\end{equation}
where $c_3$ and $c_4$ do not depend on $K,p,n$ or $r$.
\end{lemma}

Recall that $p^2 = 2l + o(l),\; l \to \infty$ to get an immediate corollary of Theorem \ref{jice_b}.
\begin{corollary}
\label{jice_cor}
If $n \geqslant p$ and $\eta \asymp \frac{\norm{\S-\Q}_2}{K}$, with high probability
\begin{equation}
\norm{\widehat{\Y} - \Y}_F \leqslant \frac{4\overline\lambda^2\sqrt{2r}}{\underline\lambda\sqrt{n}(1-\sqrt{(p-1)/n})^4}\(\sqrt{2K} +  C\sqrt{l}\).
\end{equation}
\end{corollary}
We believe that in all the obtained results the requirement $n \geqslant p$ is too strong and can be relaxed. We are currently pursuing a way to make this intuition a precise proof. 

Apparently, in the real world environment the matrix $\Q$ is not known and we need to compute the value of $\eta$ from the available data. Note that 
\begin{equation}
\underline\lambda = \min_k\lambda_p(\T_k) = \min_k\frac{1}{\norm{\Q_k}_2} = \frac{1}{\max_k \norm{\Q_k}_2},
\end{equation} 
and can naturally be estimated as
\begin{equation}
\widehat{\underline\lambda} = \frac{\(1+\sqrt{\frac{p}{n}}\)^2}{\max_k \norm{\S_k}_2},
\label{lamb}
\end{equation}
where the last relation comes from Theorem II.13 from \cite{davidson2001local}. Now Lemma \ref{r_lem} and Corollary \ref{jice_cor} together with (\ref{lamb}) suggest a natural choice of the regularization parameter value
\begin{equation}
\eta = \frac{\max_k\norm{\S_k}_2}{\sqrt{n}K(1+\sqrt{\frac{p}{n}})^2}\(\sqrt{2K} +  \sqrt{l}\).
\label{reg_cho}
\end{equation}
Below we numerically demonstrate the performance characteristics of the JICE algorithm with this choice of $\eta$.

\section{Numerical Simulations}
\subsection{Circulant Model}
For our first numerical experiment we took the circulant model. An important property of circulant matrices is that they are diagonalized by a Digital Fourier Transform (DFT) basis and, therefore, the structure is preserved under the inversion which we utilize below. In order for a real covariance matrix $\Q$ to be circulant, it should have an odd dimension $p$ and its spectrum must be of the form
\begin{equation}
\text{spec}(\Q) = [\xi_{(p+1)/2}, \xi_1,\dots, \xi_{(p-1)/2}, \xi_{(p-1)/2}, \dots, \xi_1],
\label{cir_spec}
\end{equation}
with arbitrary positive $\xi_i$-s and the order corresponding to the standard ordering of the DFT basis vectors (so that the same eigenvalues correspond to the conjugated eigenvectors). In our experiment we have generated $K=50$ random i.i.d. positive definite circulant matrices as
\begin{equation}
\Q_k = 2\I + \textbf{DFT}\;\text{spec}(\Q_k)\;\textbf{DFT}^H,
\end{equation}
where the spectrum $\text{spec}(\Q_k)$ was generated according to (\ref{cir_spec}) with $\xi_i$ i.i.d uniform over the interval $[-1,1]$ and $\textbf{DFT}$ is the unitary DFT basis matrix of dimension $p$. Figure \ref{perf_n} shows the empirical, averaged over $1000$ experiments, $\mathbf{MSE}$-s of a few estimators as functions of the number of samples $n$ in groups. The TSVD curve demonstrates the performance of the Truncated SVD after inversion (or pseudoinversion) of the group sample covariances, as proposed in the beginning of section \ref{jice_alg} and follows the development in \cite{soloveychik2015joint} applied to inverse covariances. The truncation threshold for the SVD was chosen according to the rule described in \cite{soloveychik2015joint} in such a way, that the truncated signal would convey $90\%$ of the original power. Our JICE algorithm and its bias removing modifications are named JICE and JICE\_BR, respectively. For comparison we also plot the $\mathbf{MSE}$-s of the inverse SCM (ISCM), its projection onto the known structure (Projection) and the $\mathbf{CRB}$ bound calculated in (\ref{crb_e}). The graph demonstrates that when the number of samples $n$ in each of the groups is relatively small, both JICE and JICE\_BR significantly outperform the competitors. In addition, as expected, the bias removal step improves the performance significantly and is quite close to the $\mathbf{CRB}$.

The main goal of our next simulation study was to examine how good is the choice of the regularization parameter given by (\ref{reg_cho}). For this purpose in Figure \ref{perf_n_cc} we compare the performances of the JICE\_BR algorithm in the same circulant setting with the benchmark (JICE\_BR-best). To get the JICE\_BR-best curve we took a dense grid of regularization parameters $\mathcal{Z} = \frac{1}{\sqrt{nK}}[0.05,\; 0.1,\; 0.2,\;\dots,10,\;25,\;50]$ and, given the true covariances, chose the best $\eta_b \in \mathcal{Z}$ producing the smallest $\mathbf{MSE}$. These $\mathbf{MSE}$ as functions of $n$ are depicted in the figure. As the Figure clearly suggests, the theoretically derived in (\ref{reg_cho}) value of the regularization parameter demonstrates very good performance and can be readily used in applications.

\begin{figure}[!t]
\subfloat[]{\label{perf_n}\includegraphics[height=3.4in]{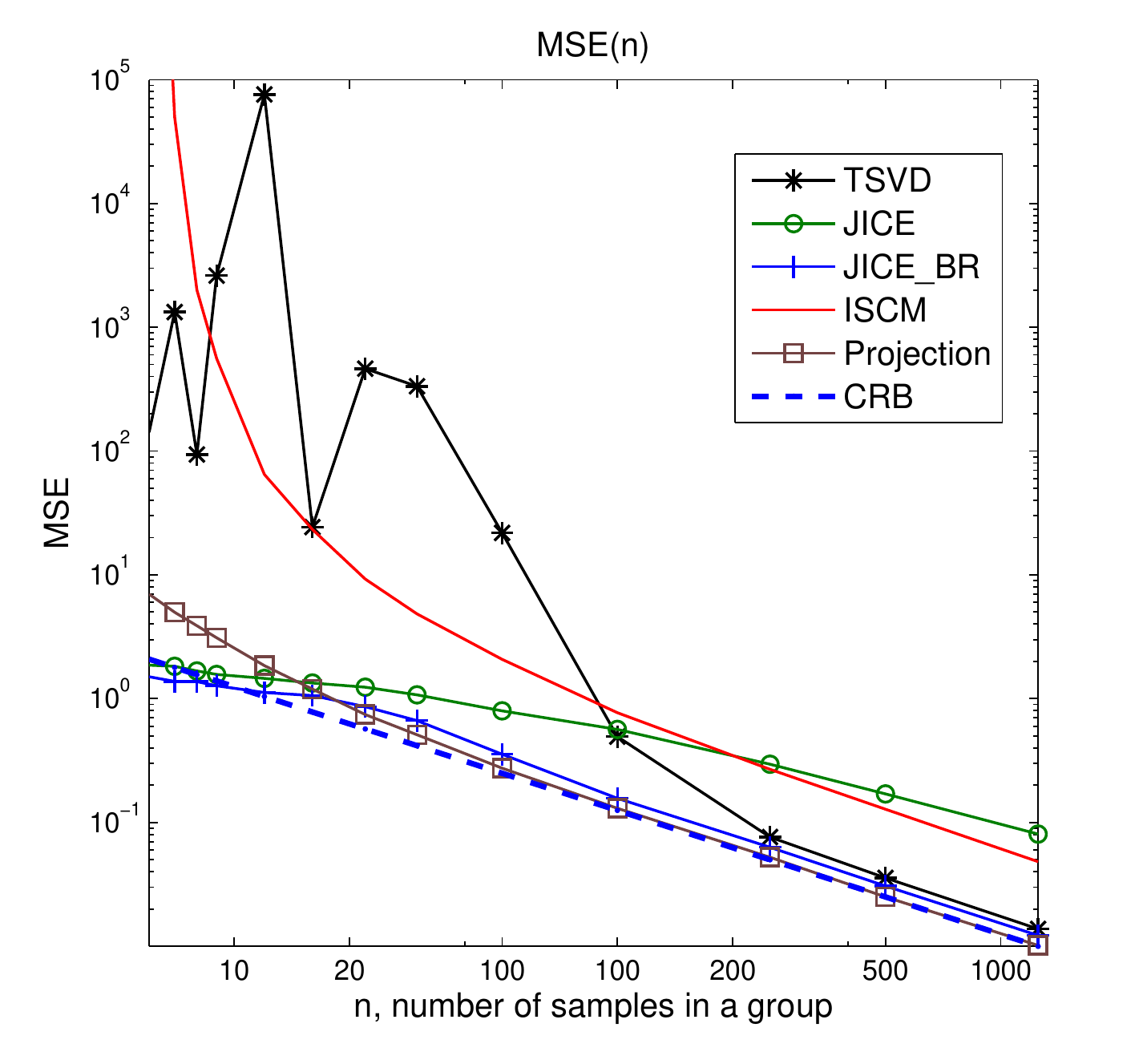}} \vspace{-0.1 cm}\\
\subfloat[]{\label{perf_n_cc}\includegraphics[height=3.4in]{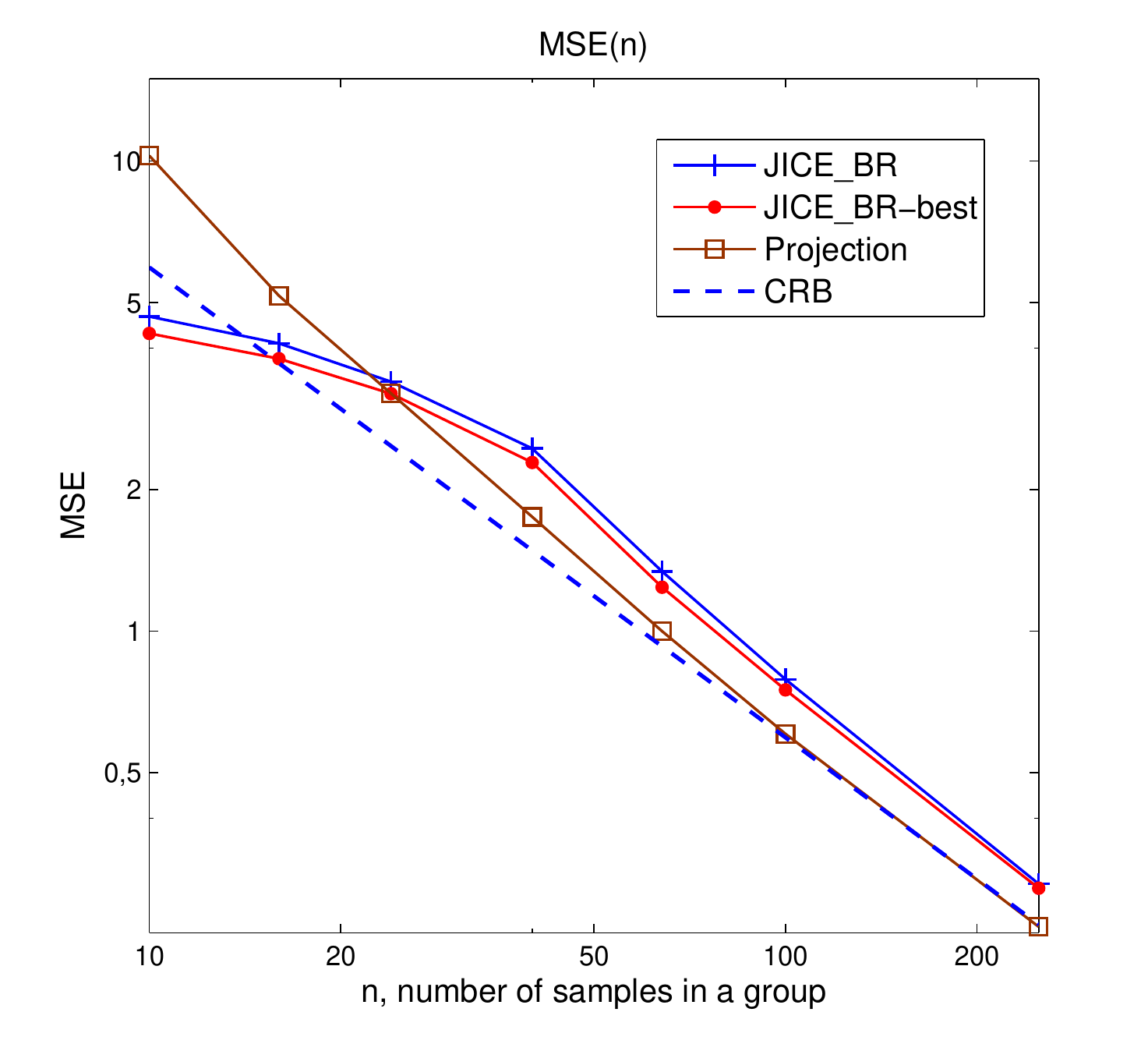}} \vspace{-0.1 cm}
\caption{JICE algorithm performance in the circulant structure case, $p=5,\; l = 15,\; r = 5,\; K = 32$.}
\end{figure}

\subsection{Graphical Models}
\begin{figure}
\centering
\includegraphics[height=3.4in]{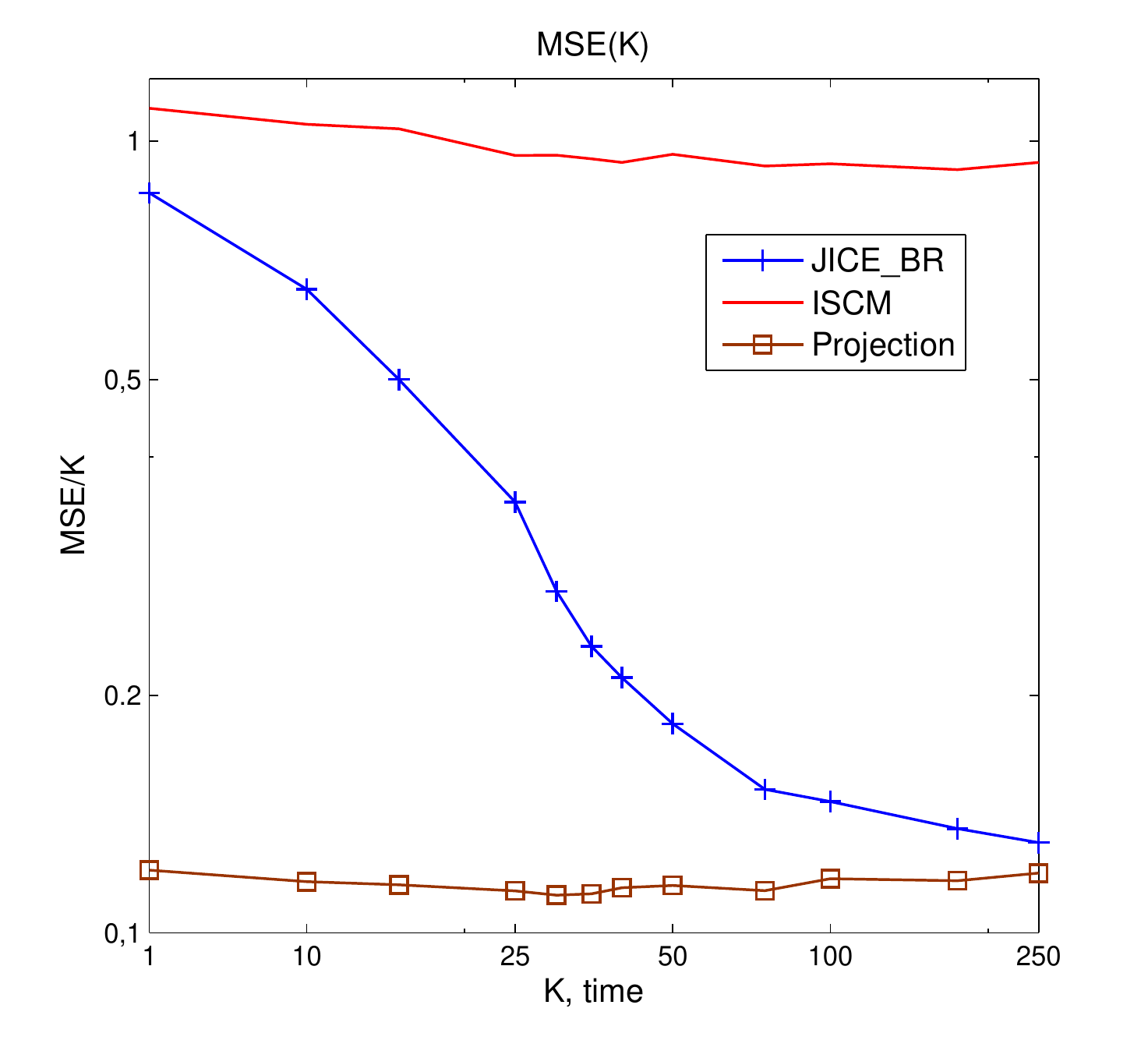}
\caption{Time-tracking learning process in Graphical Models, $p=5,\; b=1,\; l = 15,\; r = 9,\; n=40$.}
\label{perf_k}
\vspace{-0.5 cm}
\end{figure}

The goal of this section is to examine the learning abilities of the JICE technique in Gaussian Graphical Models. The setting of the simulation identifies every other group of measurements $k=1,\dots,K$ with data coming with the successive clock ticks. The inverse covariances of the groups were generated by the random process
\begin{equation}
\bar\T_t = \T_{t-1} + \alpha \bm\Delta_t,\quad \T_t = \frac{\bar\T_t}{\norm{\bar\T_t}},\quad t = 0, 1, \dots, K,
\end{equation}
where $\alpha$ is a small weight and $\bm\Delta_t$ is such a random matrix preserving the sparsity pattern that $\T_t \succ \epsilon \I,\; t = 0, 1, \dots, K$. Specifically, we fixed $p=5$ and chose the inverse covariances to be symmetric positive definite tridiagonal matrices with the smallest eigenvalue not less than $\varepsilon = 0.1$. The data generating algorithm worked as follows: on the fist step we generated symmetric tridiagonal $\T_0$ with uniformly over $[-1,1]$ distributed entries and checked if $\lambda_p(\T_0) \geqslant \varepsilon$, if this condition was violated, we repeated this step until it was met. Then on each time count $t$ we performed a similar procedure: we took symmetric tridiagonal $\bm\Delta_t$ with uniformly randomly over $[-1,1]$ entries and checked if $\lambda_p(\T_t) \geqslant \varepsilon$. If the last condition was not fulfilled iteration was repeated until it was satisfied. In our experiment $\alpha$ was taken to be $0.25$. The number of degrees of freedom in this model is $r = 9$ and the ambient dimension is $l = 15$. The number of measurements in each group was fixed to be $n=40$. Figure \ref{perf_k} shows that the JICE\_BR learns the structure with the time and effectively applies this knowledge in the estimation. All the results were averaged over $100$ iterations.

\section{Conclusion}
In this paper we consider the problem of joint inverse covariance estimation with linear structure, given heterogeneous measurements. The main challenge in this scenario is twofold. At first, the underlying structure is to be discovered and then it should be utilized to improve the concentrations estimation. We propose a novel algorithm coupling these two stages into one optimization program and propose its efficient implementation. The main aim of our current research is to provide tight upper bound guarantees on the performance of the proposed technique.

\appendix
\begin{proof}[Proof of Theorem \ref{jice_b}]
To simplify the proof, without loss of generality, we assume that $\Y$ and all its analogs are formed by $\vec{\T_k}$, rather than $\vech{\T_k}$.

Consider the first order Taylor's expansion of the target (\ref{main_alg}) with the remainder in Lagrange form in the vicinity of $\Y$
\begin{align}
\Delta f &= f(\Y+\Delta\Y) - f(\Y) \nonumber\\
&= \nabla_{\Y}f(\Delta\Y) + \frac{1}{2}\nabla^2_{\bar{\Y}}f(\Delta\Y) + \eta\Delta\norm{\Y}_*,
\label{taylor}
\end{align}
where the derivatives are treated as linear and quadratic forms correspondingly, their subscripts denote the point of evaluation, and
\begin{align}
\Delta\Y &= [\vec{\Delta\T_1},\dots,\vec{\Delta\T_K}], \\
\bar{\Y} &= \Y+\alpha\Delta\Y,\;\; \alpha \in [0,1], \\
\Delta\norm{\Y}_* &= \norm{\Y+\Delta\Y}_*-\norm{\Y}_*.
\end{align}
Compute the derivatives
\begin{align}
\nabla_{\Y}f(\Delta\Y) &= \frac{1}{K}\sum_{k=1}^K \Tr{(\S_k-\T_k^{-1})\Delta\T_k} \nonumber \\ 
&= \frac{1}{K}\Tr{(\S-\Q)^T\Delta\Y}, \\
\nabla^2_{\bar{\Y}}f(\Delta\Y) &= \frac{1}{K}\sum_{k=1}^K \Tr{\[(\T_k+\alpha\Delta\T_k)^{-1}\Delta\T_k\]^2}.
\label{ders}
\end{align}
Consider the target function over the set satisfying
\begin{equation}
\lambda_p\((\T_k+\alpha\Delta\T_k)^{-1}\) \geqslant c_5\lambda_p(\S_k),\; k=1,\dots,K.
\end{equation}
Theorem 1.1 from \cite{rudelson2009smallest} claims that when $n \geqslant p$, with probability at least $1-2^{-n+p-1}-e^{-c_1n}$,
\begin{multline}
\lambda_p(\S_k) \geqslant c_6\lambda_p(\Q_k)\frac{(\sqrt{n}-\sqrt{p-1})^2}{n} \\ \geqslant \frac{c_6}{\overline\lambda}\(1-\sqrt{\frac{p-1}{n}}\)^2,
\label{lob}
\end{multline}
where $c_1$ and $c_6$ are universal constants. The union bound implies that with probability at least $1-K\(2^{-n+p-1}-e^{-c_1n}\)$ (\ref{lob}) holds simultaneously for all $k$ and, therefore, after denoting $c_2 = c_5c_6$, with the same probability
\begin{equation}
\lambda_p\((\T_k+\alpha\Delta\T_k)^{-1}\) \geqslant \frac{c_2}{\overline\lambda}\(1-\sqrt{\frac{p-1}{n}}\)^2.
\label{tbo}
\end{equation}
over the set under investigation. Denote 
\begin{equation}
\zeta = \frac{c_2}{\overline\lambda}\(1-\sqrt{\frac{p-1}{n}}\)^2,
\end{equation}
and plug (\ref{tbo}) into (\ref{ders}) to get
\begin{equation}
\nabla^2_{\bar{\Y}}f(\Delta\Y) \geqslant \frac{\zeta^2}{K}\sum_{k=1}^K \norm{\Delta\T_k}_F^2 = \frac{\zeta^2}{K}\norm{\Delta\Y}_F^2.
\end{equation}
At the optimal point $\widehat{\Y}$, necessarily $\Delta f \leqslant 0$, thus 
\begin{align}
\frac{1}{K}\Tr{(\S-\Q)^T\Delta\Y} + \frac{\zeta^2}{2K}\norm{\Delta\Y}_F^2 + \eta\Delta\norm{\Y}_* \leqslant 0,
\label{df_in}
\end{align}
or
\begin{equation}
\frac{\zeta^2}{2}\norm{\Delta\Y}_F^2 \leqslant - K\eta\Delta\norm{\Y}_* - \Tr{(\S-\Q)^T\Delta\Y}.
\label{iin}
\end{equation}

Introduce the SVD of $\Y$
\begin{equation}
\Y = \U\bm\Sigma\V^T = [\U_r \U_r^\perp]\bm\Sigma[\V_r \V_r^\perp]^T,
\end{equation}
where $\U\in\mathbb{R}^{p^2 \times p^2}, \V\in\mathbb{R}^{K \times K}$ are orthogonal, $\U_r\in\mathbb{R}^{p^2 \times r}, \V_r\in\mathbb{R}^{K \times r}$ and the upper left part of the diagonal of $\bm\Sigma \in \mathbb{R}^{p^2 \times K}$ contains the $r$ non zero singular values of $\Y$. Let
\begin{equation}
\bm\Gamma = \U^T\Delta\Y\V = \begin{pmatrix}
  \bm\Gamma_{11} & \bm\Gamma_{12}\\
  \bm\Gamma_{21} & \bm\Gamma_{22}\\
  \end{pmatrix},
\end{equation}
where $\bm\Gamma_{11} \in \mathbb{R}^{r \times r}$ and define
\begin{equation}
\Delta\Y'' = \U \begin{pmatrix}
  0 & 0\\
  0 & \bm\Gamma_{22}\\
  \end{pmatrix}\V^T,
\label{d1}
\end{equation}
\begin{equation}
\Delta\Y' = \Delta\Y - \Delta\Y''.
\label{d2}
\end{equation}

\begin{lemma} 
\label{lem_au1}
With the notation (\ref{d1}), (\ref{d2}),
\begin{equation}
\label{lemst}
- \Delta\norm{\Y}_* \leqslant \norm{\Delta\Y'}_* - \norm{\Delta\Y''}_*.
\end{equation} 
\end{lemma}
\begin{proof}
The proof can be found below.
\end{proof}

Plug (\ref{lemst}) into (\ref{iin}) and use Holder's inequality
\begin{equation}
\Tr{(\S-\Q)^T\Delta\Y} \leqslant \norm{\S-\Q}_2\norm{\Delta\Y}_*,
\end{equation}
to obtain
\begin{multline}
\frac{\zeta^2}{2}\norm{\Delta\Y}_F^2 \leqslant - K\eta\Delta\norm{\Y}_* - \Tr{(\S-\Q)^T\Delta\Y} \\
\leqslant K\eta(\norm{\Delta\Y'}_* - \norm{\Delta\Y''}_*) + \norm{\S-\Q}_2(\norm{\Delta\Y'}_* + \norm{\Delta\Y''}_*),
\end{multline}
where in the last transition we have applied the triangle inequality. Recall that $\eta \geqslant \norm{\S-\Q}_2/K$ to obtain 
\begin{equation}
\frac{\zeta^2}{2}\norm{\Delta\Y}_F^2 \leqslant 2K\eta\norm{\Delta\Y'}_*.
\end{equation}
Note that
\begin{align}
&\rank(\Delta\Y') = \rank{\begin{pmatrix}
  \bm\Gamma_{11} & \bm\Gamma_{12}\\
  \bm\Gamma_{21} & 0\\
  \end{pmatrix}} \nonumber\\
&\leqslant \rank{\begin{pmatrix}
  \bm\Gamma_{11} & \bm\Gamma_{12}\\
  0 & 0\\
  \end{pmatrix}} + \rank{\begin{pmatrix}
  \bm\Gamma_{11} & 0\\
  \bm\Gamma_{21} & 0\\
  \end{pmatrix}} \leqslant 2r,
\label{rankin}
\end{align}
and use the following norm relation
\begin{equation}
\norm{\Delta\Y'}_* \leqslant \sqrt{\rank(\Delta\Y')}\norm{\Delta\Y'}_F,
\end{equation} 
together with the decomposition (\ref{d1})-(\ref{d2}) to get
\begin{multline}
\frac{\zeta^2}{2}\norm{\Delta\Y}_F^2 \leqslant 2K\eta\norm{\Delta\Y'}_* \\ \leqslant 2K\eta\sqrt{2r}\norm{\Delta\Y'}_F \leqslant 2\sqrt{2r}K\eta\norm{\Delta\Y}_F.
\end{multline}
We finally achieve the desired bound
\begin{equation}
\norm{\Delta\Y}_F \leqslant \frac{4\sqrt{2r}K\eta}{\zeta^2} = \frac{4\overline\lambda^2\sqrt{2r}K\eta}{c_2^2(1-\sqrt{(p-1)/n})^4}.
\end{equation}
\end{proof}

\begin{proof}[Proof of Lemma \ref{lem_au1}]
\begin{align}
\norm{\Y + \Delta\Y''}_* &= \norm{\U_r\U_r^T\Y\V_r\V_r^T + \Delta\Y''}_* \nonumber\\
&= \norm{\Y}_* + \norm{\Delta\Y''}_*.
\end{align}
Together with the triangle inequality this yield
\begin{align}
\norm{\Y+\Delta\Y}_* &= \norm{\(\Y+\Delta\Y''\) + \Delta\Y'}_* \nonumber\\
&\geqslant \norm{\Y+\Delta\Y''}_* - 
\norm{\Delta\Y'}_* \nonumber\\
&\geqslant \norm{\Y}_*+ \norm{\Delta\Y''}_* -\norm{\Delta\Y'}_*,
\end{align}
and, finally,
\begin{equation}
- \Delta\norm{\Y}_* = - \norm{\Y+\Delta\Y}_* + \norm{\Y}_* \leqslant \norm{\Delta\Y'}_* - \norm{\Delta\Y''}_*.
\label{aui}
\end{equation}
\end{proof}

\bibliographystyle{IEEEtran}
\bibliography{icls_j.bbl}

\end{document}